%% file: main.tex
\documentclass{article}
\usepackage{ijcai26}

\usepackage{times}
\usepackage{soul}
\usepackage{url}
\usepackage[hidelinks]{hyperref}
\usepackage[utf8]{inputenc}
\usepackage[small]{caption}
\usepackage{graphicx}
\usepackage{amsmath}
\usepackage{amsthm}
\usepackage{booktabs}
\usepackage{algorithm}

\usepackage[noend]{algpseudocode}
\usepackage{tikz}
\usetikzlibrary{shapes.geometric}
\usepackage{amssymb}
\usepackage{dsfont}
\usepackage[disableredefinitions,basic]{complexity}
\usepackage{nicefrac}

\newtheorem{example}{Example}
\newtheorem{definition}{Definition}

\newtheorem{theorem}{Theorem}

\newcommand{\rational}{{\mathds Q}}

\newcommand{\disj}{\vee}
\newcommand{\conj}{\wedge}
\newcommand{\impl}{\rightarrow}

\newcommand{\setd}{\backslash}

\newcommand{\paths}{\ensuremath{\mathcal{P}}}
\newcommand{\seq}{\ensuremath{\mathit{s}}}
\newcommand{\pathassign}[1]{\ensuremath{\pi_{#1}}}

\newcommand{\instance}{\ensuremath{I=(\phi,X,P,C)}}
\newcommand{\dt}[1]{\ensuremath{G_{#1}}}
\newcommand{\valuation}{\ensuremath{\pi}}

\newcommand{\cost}{\ensuremath{c}}

\newcommand{\curcost}{\ensuremath{parentcost}}
\newcommand{\mincost}{\ensuremath{childcost}}
\newcommand{\LB}{\ensuremath{bound}}
\newcommand{\prob}{\ensuremath{p}}

\newcommand{\selcost}{\ensuremath{C_x}}
\newcommand{\continue}{\textbf{continue}}

\newcommand{\FlagExact}{\ensuremath{\textsc{Exact}}}
\newcommand{\FlagLB}{\ensuremath{\textsc{LB}}}

\newcommand{\sharpP}{\ComplexityFont{\#P}}
\newcommand{\FPSPACE}{\ComplexityFont{FPSPACE}}

\title{Cost-Optimal Decision Diagrams for Stochastic Boolean Function Evaluation}

\author{
Xia Zong$^1$
\and
Tuomo Lehtonen$^1$
\and
Jussi Rintanen$^1$\\
\affiliations
$^1$Department of Computer Science, Aalto University\\
\emails
\{xia.zong@aalto.fi, tuomo.lehtonen@aalto.fi, jussi.rintanen@aalto.fi\}
}
\begin{document}

\maketitle

\begin{abstract}
In many decision-making scenarios, acquiring information incurs different costs. We consider the problem of constructing a deterministic evaluation strategy that minimizes the expected cost of evaluating a propositional formula under variable costs and a probability distribution over truth assignments. We present a branch-and-bound algorithm with variable-selection heuristics, pruning, and caching. To the best of our knowledge, it is the first practical exact algorithm for this level of generality. Experiments on random instances demonstrate scalability and quantify the efficiency--quality trade-off of a greedy beam-search variant. We additionally evaluate a structured heart-disease diagnosis instance. Finally, we prove that the problem is $\sharpP$-hard and contained in $\PSPACE$.
\end{abstract}

\section{Introduction}\label{sec:intro}

Evaluating a Boolean formula becomes non-trivial when observing its variables incurs different costs, especially when the variables are statistically dependent. This scenario is typical in applications such as diagnosis, testing and classification~\cite{Unluyurt04,Unluyurt25,Moret82}. This problem is known as the {\em Stochastic Boolean Function Evaluation} (SBFE) problem, and in the sequential testing literature~\cite{Unluyurt04} {\em sequential testing of Boolean functions}.

Exhaustive brute-force methods have been proposed for some classes of formulas~\cite{Halpern74,BreitbartReiter75}.
More recent works focus on worst-case complexity and approximation algorithms~\cite{DeshpandeHK14}, sometimes in special cases such as symmetric functions~\cite{GkenosisGHK22}, or using costs but no probabilities or assuming uniform probabilities for all variables~\cite{CharikarFGKRS00,BlancLT21}.
Polynomial-time cases with restricted Boolean functions such as $k$-out-of-$n$~\cite{SalloumBreuer84,ChangSF90} or disjunctions (analogously conjunctions) of atoms~\cite{Garey73} have been considered for sequential testing.

The SBFE problem is related to various decision-tree problems~\cite{GareyGraham74,HyafilRivest76,Quinlan86}.
Each node of a decision tree represents a test (possibly with an associated cost), the result of which determines which sub-tree to follow, and leaf nodes contain a value (such as a class or diagnosis) to which the sequence of test results from the root node is mapped.

Probabilities may be associated with the classes or the test outcomes in different ways. The tests and their relation to the diagnoses or the test results may be expressible as a mapping or a formula.

While there is a continuing interest in analytical works on tractable approximations and special cases, there has been little focus on scalable optimal methods for the SBFE problem in the most general setting with general Boolean functions and complex probability distributions on the variables.
This is the gap our work helps to fill by proposing and evaluating a practical algorithm that incorporates heuristics and effective pruning methods to improve scalability.

We address the off-line SBFE problem, with the objective of constructing decision diagrams to evaluate Boolean formulas such that the expected cost of evaluation is minimal.
The diagram represents a general procedure that can be used multiple times with different variable values, as long as the costs and the statistical dependencies between the variable values remain the same.

Concretely, we consider a general problem of optimal decision diagram construction, starting from an arbitrary propositional formula $\phi$.
The value of each variable in $\phi$ can be observed, incurring a cost, and the conditional probability of a variable being true, given some subset of others, is known.
We want to minimize the expected cost of determining the truth-value of $\phi$.
Additionally, in contrast to much of the earlier work that has imposed restrictions on cost functions or probability distributions, we permit arbitrary positive costs and general probability distributions that can, for example, be represented by Bayesian networks.

Our contributions are as follows.
First, we propose a practical branch-and-bound algorithm for constructing optimal decision diagrams, including various variable ordering heuristics and efficient pruning and caching schemes.
We present an implementation of our algorithm and evaluate its runtime performance.
We empirically show the efficiency improvements from pruning, caching and variable ordering heuristics, and show performance variation across different clause-to-variable ratios, specifically a region around which instances are particularly hard.
We also investigate the trade-off between runtime and quality when using a beam search instead of optimal search.
Finally, we present novel complexity results for constructing optimal decision diagrams for arbitrary Boolean formulas, showing the problem to be $\sharpP$-hard.
NP-hardness of various SBFE variants has been stated in the literature~\cite{AllenHKU17,GreinerHJM06,KaplanKM05}, but stricter lower bounds have not been reported for the SBFE problem or the equivalent sequential testing problem~\cite{Unluyurt25}.

We review the related literature in more detail in Section~\ref{sec:related}.
The problem is formally defined in Section~\ref{sec:problem_statement}, and
its complexity is analyzed in Section~\ref{sec:complexity}.
We present our algorithm in Section~\ref{sec:algo} and
empirically evaluate its implementation in Section~\ref{sec:experiments}.

\section{Related Work}
\label{sec:related}

A number of works in the literature consider the evaluation of a Boolean condition or a classifier, in some cases with a decision diagram or tree as a result.

\paragraph{Binary Decision Diagrams} Evaluation of Boolean functions can be described with different types of Binary Decision Diagrams, including Free BDDs (FBDD), which allow different variable orderings on different paths through the diagram. Methods for finding FBDDs of minimal size have been developed~\cite{GuntherDrechsler02}.
This is a restricted setting compared to ours: FBDD variables have no costs or probabilities for different value assignments, and the objective is to minimize size rather than the cost of using variables.

\paragraph{Object identification problems} The objective is to identify the class a given object belongs to, by performing a sequence of tests, each of which corresponds to a subset of objects, with the outcome being whether the object belongs to that subset~\cite{GareyGraham74,HyafilRivest76}. All possible test sequences are represented as a decision tree. Probabilities, if used at all, are assigned to classes only, not to tests.
Hyafil and Rivest~\shortcite{HyafilRivest76} show that checking the existence of a decision tree of a given size is \NP-complete.
Later works are analytical and focus on approximation algorithms with logarithmic approximation ratios~\cite{ChakaravarthyPRS09,LiLM20,ZhuoNagarajan25}.

This problem is simpler than the one in our work due to the restricted probabilities and propositional language.
The case with only two classes is a special case of our problem.
The problem studied in our work can be generalized to multi-class classification, for example via vectors of Boolean formulas.

\paragraph{Binary classification with Boolean formulas} Using decision trees to determine the value of a Boolean formula, similarly to the present work, has been considered in the literature.
Reinwald and Soland~\shortcite{ReinwaldSoland66} construct decision trees with the least expected cost for decision tables, which are essentially Boolean functions in disjunctive normal form, with a probability associated to each disjunct.
Breitbart and Reiter~\shortcite{BreitbartReiter75} present an algorithm for finding a decision tree with the minimum average number of observations to determine the value of a monotonic Boolean function, with unit costs for every observation, and uniform probability for every valuation. The treatment of probabilities here is far less general than in our work, and the algorithm for finding optimal decision trees is brute-force without heuristics or pruning methods.
Many works on circuit complexity consider the depth of circuit diagrams for evaluating Boolean functions~\cite{Wegener84}, without separately considering costs or associating probabilities with different value combinations.
Approximation algorithms for the SBFE problem have been investigated~\cite{DeshpandeHK14,GkenosisGHK22}.

\paragraph{Decision trees in Machine Learning}
Decision trees have been used for representing large datasets more compactly, to make the data more explainable, and to turn it into an effective classification procedure~\cite{Quinlan79,Quinlan86,BertsimasDunn17,MienyeJere24}.
Costs are sometimes considered, for example in algorithms that attempt to minimize the cost of misclassifying instances, sometimes together with the cost of reading the variable values~\cite{LingYWZ04,LomaxVadera13}.
The machine learning problem is fundamentally different: learning (an approximation of) some decision function, in contrast to the SBFE problem which aims to devise a general strategy of most cheaply evaluating the decision according to a given decision function.
This optimality in evaluation is not considered nor solved by learning a classification function.

\section{Problem Statement}
\label{sec:problem_statement}

\newcommand{\formula}{\phi}
\newcommand{\clause}{C}
\newcommand{\vars}{\mbox{Vars}}
\newcommand{\lits}{\mbox{Lits}}
\newcommand{\model}{\tau}

We briefly recall the basics of propositional logic.
Let $X = \{ x_1,\ldots,x_n\}$ be a set of propositional variables.
Formulas are formed from variables in $X$ and connectives $\disj$, $\conj$ and $\neg$.
A valuation $v : X\rightarrow\{0,1\}$ is a total function $X$
to values 0 ({\em false}) and 1 ({\em true}), and it determines the truth
of formulas $\phi$, indicated by $v\models\phi$, as follows.
For atomic formulas $x\in X$, $v\models x$ if and only if $v(x) = 1$. For compound formulas,
$v\models\phi_1\conj\phi_2$ iff $v\models\phi_1$ and $v\models\phi_2$,
$v\models\phi_1\disj\phi_2$ iff $v\models\phi_1$ or $v\models\phi_2$,
and $v\models\neg\phi$ iff $v\models\phi$ does not hold.
A formula $\phi$ is {\em satisfiable} if $v\models\phi$ for at least one valuation $v$.
It is {\em valid} if $v\models\phi$ for all valuations $v$.

Consider a formula $\phi$ with Boolean variables in $X = \{ x_1,\ldots,x_n\}$,
and the problem of determining whether $\phi$ is true or false.
Each variable $x\in X$ can be observed at some positive rational cost $\cost(x)$, $\cost : X\rightarrow\rational$.
Further, the probability model provides $P(v(x)=b \mid B)$ for each variable $x\in X$, value $b\in\{0,1\}$, and partial valuation $B$ over variables in $X\setd\{x\}$ reached during evaluation. These conditionals are induced by a joint distribution over $X$.

\begin{definition}[Problem instance]\label{de:infcollection}
    Our problem instance is \instance, with

a Boolean formula $\phi$ over a set of variables $X$;
a conditional probability distribution $P$ on $X$; and
a cost function $\cost : X\rightarrow\rational$ of observing variables.

\end{definition}

The aim is to construct a minimal cost decision diagram that unambiguously determines the value of $\phi$.
Intuitively, given a Boolean formula $\phi$ over variables $X$, we say that \dt{\phi} is a decision diagram for $\phi$ if it is a decision diagram with non-leaf nodes labeled with variables in $X$ and arcs with a value for the corresponding variable, and the following conditions hold.
Each possible valuation is represented by exactly one path in the diagram.

Formally, we require the following.
\begin{definition}[Decision diagram for a formula]\label{de:dtrees}
Given
    an instance \instance, a {\em decision diagram for $\phi$} is a directed acyclic graph $G = \langle N,A\rangle$ with
$N$ a finite set of nodes and $A\subseteq N\times N$ a set of arcs, where
\begin{enumerate}
\item $O : N\rightarrow X\cup\{ \top, \bot\}$ assigns to each node a variable, the symbol $\top$ (for {\em true}), or $\bot$ (for {\em false}),
with
\begin{itemize}
\item $O(n)\in\{\top,\bot\}$ only if $n$ has no outgoing arcs,
\item $O(n)\in X$ only if $n$ has two outgoing arcs, and
\end{itemize}
\item $V : A\rightarrow\{0,1\}$ indicates whether the arc is followed when the value
of the observed variable is {\em false} or {\em true},
\item for every node $n$ such that $O(n)\in X$, there is exactly one outgoing arc $a$
labeled $V(a)=0$ and exactly one outgoing arc $a'$ labeled $V(a')=1$,
\item
there is a root node $r\in N$ such that no arc ends in $r$,
\item every node in $N\setd\{r\}$ has at least one incoming arc.

\end{enumerate}

    Consider a path starting from the root of $G$, $\seq=(a_1,\dots,a_m)$ where each $a_i\in A$.
    Let $(n_1,\dots,n_m)$ be the sequence of nodes, not including $\top$ or $\bot$, appearing in the arcs of $\seq$.
    We define $\pathassign{\seq}$, the valuation function corresponding to $\seq$, as $\pathassign{\seq}(O(n_i))=V(a_i)$ for all $1\leq i \leq m$.

    $G$ is a decision diagram for $\phi$ (denoted $\dt{\phi}$) if moreover
    1) for all paths $\seq$ of \dt{\phi}, if $\seq$ ends in $\top$, then $\phi$ is valid under the valuation $\pathassign{\seq}$, and if $\seq$ ends in $\bot$, then $\phi$ is unsatisfiable under $\seq$;
    and 2) for each total valuation \valuation\ over $X$, there is a path $\seq$ in \dt{\phi} such that $\pathassign{\seq}$ can be extended to \valuation.
\end{definition}

In words, a decision diagram for a formula $\phi$ classifies any valuation as $\top$ ($\bot$) iff $\phi$ is true (false) under (all extensions to) that valuation.
A decision tree is a special case of a decision diagram, where each node has at most one incoming arc.

We move on to defining the (evaluation) cost and thereby the optimality of decision diagrams for formulas.
Let $\paths(\dt{\phi})$ denote the collection of paths with non-zero length starting from the root in a decision diagram $\dt{\phi}$.
For brevity, for path $\seq$, let $\pathassign{\seq}^j$ denote $\pathassign{\seq}$ restricted to the first $j$ valuations in the path, i.e. if the nodes corresponding to $\seq$ are $(n_1^s,\dots,n_m^s)$, let $N'=\{n_i^s\in N\mid 1\leq i\leq j\}$ and $\pathassign{\seq}^j=\pathassign{\seq}|_{N'}$.
Further, let $\pathassign{\seq}(j)$ denote the $j$th valuation in $\seq$, i.e. $O(n_i)=\pathassign{\seq}(O(n_i))$.

\begin{definition}[Optimal decision diagram for $\phi$]
    \label{de:dt_for_phi}
    Given an instance \instance\ and a decision diagram for $\phi$, $\dt{\phi}$, the expected (evaluation) cost of \dt{\phi} is
\begin{align*}
    \cost(r)+ \sum_{\seq\in\paths(\dt{\phi})} \left( \cost(n^{\seq}_{|\seq|}) \cdot  \prod_{j=1}^{|\seq|} P(\pathassign{\seq}(j)\mid \pathassign{\seq}^{j-1})\right).
\end{align*}
    The product of conditional probabilities is the probability, under the joint distribution, of reaching the corresponding path.
    A decision diagram for $\phi$ is optimal if it has the minimal cost of all decision diagrams for $\phi$.
\end{definition}

In words, for each path with length greater than $0$, the cost of observing the final variable in the path is multiplied by the probability of this particular valuation, and the cost of the whole tree is the cost of observing the root variable plus the sum of all paths.
The cost is the expected cost of determining the truth value of $\phi$ by a sequence of variable observations.

\begin{example}
    Consider the instance $\instance$ with $\phi=(A\disj B\disj C)\conj(\neg A\disj D)$, $\cost(A)=\cost(B)=\cost(D)=1$ and $\cost(C)=10$, and uniform probabilities, i.e. each variable is true with a probability of 0.5.
    Figure~\ref{fi:dt_ex} shows two decision diagrams for $I$ as trees; it is easy to see that they both correctly classify all valuations to $\phi$.

    The graph on the left has a cost of $1 + 1\cdot 0.5 + 1\cdot 0.5 + 10\cdot 0.5^2 = 4.5$.
    The sum terms correspond to the costs of the variables $A$, $D$, $B$, and $C$ respectively, multiplied by the probability of the evaluation reaching their respective node.
    The graph on the right has a cost of $1 + 1\cdot 0.5 + 10 \cdot 0.5 + 1 \cdot 0.5^2 = 6.75$,

    reflecting the fact that the costly $C$ is evaluated before $B$ in the right-hand branch.
\end{example}

\begin{figure}
    \centering
\begin{tikzpicture}[xscale=1.1, yscale=0.8, scale=0.9]

  \node[circle,draw,inner sep=2pt](A) at (1,4) {\small$A$};
    \node[circle,draw,inner sep=2pt](At) at (-0,3) {\small$D$};
    \node[circle,draw,inner sep=2pt](Af) at (2,3) {\small$B$};
    \node[circle,draw,inner sep=2pt](AfBf) at (2.5,2) {\small$C$};

  \node(AtDt) at (-0.5,2) {$\top$};
  \node(AtDf) at (0.5,2) {$\bot$};

  \node(AfBt) at (1.5,2) {$\top$};

  \node(AfBfCt) at (2,1) {$\top$};
  \node(AfBfCf) at (3,1) {$\bot$};

  \tikzstyle{every path}=[arrows=-stealth,->]

  \draw(A) -- (At);
  \draw[dashed] (A) -- (Af);

  \draw(At) -- (AtDt);
    \draw(At)[dashed] -- (AtDf);

  \draw(Af) -- (AfBt);
\draw(Af)[dashed] -- (AfBf);

  \draw(AfBf) -- (AfBfCt);
  \draw[dashed] (AfBf) -- (AfBfCf);

\end{tikzpicture}
    \hspace{10pt}
\begin{tikzpicture}[xscale=1.1, yscale=0.8, scale=0.9]

  \node[circle,draw,inner sep=2pt](A) at (1,4) {\small$A$};
    \node[circle,draw,inner sep=2pt](At) at (-0,3) {\small$D$};
    \node[circle,draw,inner sep=2pt](Af) at (2,3) {\small$C$};
    \node[circle,draw,inner sep=2pt](AfBf) at (2.5,2) {\small$B$};

  \node(AtDt) at (-0.5,2) {$\top$};
  \node(AtDf) at (0.5,2) {$\bot$};

  \node(AfBt) at (1.5,2) {$\top$};

  \node(AfBfCt) at (2,1) {$\top$};
  \node(AfBfCf) at (3,1) {$\bot$};

  \tikzstyle{every path}=[arrows=-stealth,->]

    \draw(A) -- (At);
  \draw[dashed] (A) -- (Af);

  \draw(At) -- (AtDt);
    \draw(At)[dashed] -- (AtDf);

  \draw(Af) -- (AfBt);
\draw(Af)[dashed] -- (AfBf);

  \draw(AfBf) -- (AfBfCt);
  \draw[dashed] (AfBf) -- (AfBfCf);

\end{tikzpicture}
\caption{Two decision trees for $(A\disj B\disj C)\conj(\neg A\disj D)$. Solid arcs evaluate the variable to 1 and dashed to 0.
}
\label{fi:dt_ex}
\end{figure}
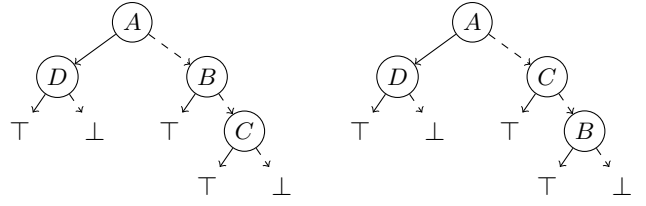

\section{Complexity}
\label{sec:complexity}

Interestingly, the complexity of the SBFE problem and related problems has not been analyzed beyond showing NP-hardness.

With dependencies between the variables and the need to minimize the {\em expected cost} of formula evaluation, the complexity increases substantially beyond NP. We address this gap in the literature.

We show that constructing a decision diagram with minimal expected cost is \sharpP-hard, i.e. as hard as many probabilistic inference problems and model-counting problems for propositional logic~\cite{Valiant79,Roth96}.
This shows, for example, that a reduction to Mixed Integer-Linear Programming, an NP-complete problem, is not feasible (under common assumptions such as NP$\subset$\sharpP).

\begin{theorem}\label{th:sharpphardness}
The problem of calculating the expected cost of the least expected cost decision diagram is \sharpP-hard.
\end{theorem}

\begin{proof}
We reduce $\#$SAT to calculating the cost of the optimal decision diagram.
Let $\phi$ be any propositional formula, and let $X$ be the set of atomic
propositions in $\phi$, with $n = |X|$.
We construct a formula $\phi'$ and a distribution $D$ so that
$\phi$ has $N$ models if and only if the expected cost $c$ of the optimal
decision diagram for $\phi'$ with data $D$ satisfies $N\leq c < N+1$.

Let $\phi' = \phi\conj z$, where $z\not\in X$.
Let $\cost(z) = 2^n$ and $\cost(x) = \nicefrac{1}{2n}$ for all $x\in X$.
Let $D$ be the uniform distribution, i.e., every valuation of $X$
has probability $2^{-n}$.
Clearly, these distributions can be represented polynomially wrt $|X|$ in many formalisms, including Bayesian networks.

In every root-to-leaf path of the decision diagram for $\phi\conj z$ and $D$, some member of $X$
is observed, and only after $\phi$ has been determined to be {\em true},
the value of $z$ is observed.
Paths that correspond to valuations that make $\phi$ false do not observe $z$,
and do not correspond to satisfying valuations of $\phi$.

Consider a node in the decision diagram corresponding to a partial valuation that assigns truth values
to variables $x_1,\ldots,x_k$, and which observes $z$.
This node corresponds to $2^{n-k}$ satisfying valuations of $\phi$.
The observation of $z$ contributes expected value $2^{-k}\cdot 2^n = 2^{n-k}$
to the expected cost of the optimal decision diagram.
Thus the sum of the contributions of the observations of $z$ in all nodes
of the optimal decision diagram equals the number of satisfying valuations $N$ of $\phi$.
Additionally, the graph contains at most $2^n$ observations for variables in $X$,
with a total maximum cost of $n\cdot \nicefrac{1}{2n} = 0.5$ on any execution of the decision diagram.
    Hence, the expected cost $c$ of the optimal decision diagram satisfies $N\leq c \leq N+0.5$.
\end{proof}

Note that model counting is \sharpP-hard even with monotone formulas~\cite{Roth96},
and as our proof preserves monotony,
constructing optimal decision diagrams for monotone formulas is also \sharpP-hard.
While a \sharpP-membership does not seem likely, we can show that computing the expected cost of the decision diagram is in \FPSPACE, and determining whether the expected cost is within given bounds is in \PSPACE.
\PSPACE\ is very close to \sharpP, as $\PH\subseteq \cP^{\sharpP}\subseteq \PSPACE$~\cite{Toda91}, so only a small gap remains in the complexity results.

\begin{theorem}
    The problem of calculating the expected cost of the optimal decision diagram is in \FPSPACE, and the problem of deciding if there is a decision diagram with an expected cost bounded by some constant is in \PSPACE.
\end{theorem}

\begin{proof}
Consider our Algorithm~\ref{alg:main} in Section~\ref{sec:algo}.
The procedure call \textsc{OptCost}($\phi,\emptyset,0,\infty,1.0)$ returns the expected cost of
the decision diagram with the lowest expected cost.
At each recursive call to \textsc{OptCost}, the number of occurrences of
variables in $X$ in the formula $\phi$ is decreased by at least one.
Hence, the recursion depth of \textsc{OptCost} is linear in $|X|$.

For the representations of the statistical data we consider, the computation of
    expected cost on lines~\ref{li:cost2} and \ref{li:cost3} takes polynomial space, as the probabilistic inference problem for Bayesian networks is solvable in \PSPACE.
Hence, the memory consumption at any point of the search graph is polynomial,
ignoring caching (which is not needed for the correctness or completeness
of the algorithm.)
This shows the \FPSPACE\ membership for the function problem as well as the \PSPACE\ membership for the decision problem.
\end{proof}

\section{Algorithm for Optimal Decision Diagrams}
\label{sec:algo}

We present Algorithm~\ref{alg:main} for finding a decision diagram of minimum expected cost for an instance \instance.
We employ a bounding method (lines~\ref{li:prune1}, \ref{li:prune2}, and \ref{li:prune3}) and caching of exact subproblem solutions (lines~\ref{li:cache_check} and \ref{li:cache_add}).
Both result in substantial pruning of the search tree. The advanced lower-bound cache is given in Appendix~\ref{sec:advanced_cache}.
The pseudocode computes only the optimal cost; simple book-keeping can be added to construct the optimal decision diagram.

The algorithm proceeds by iteratively choosing a variable $x$ to observe, and
then constructing sub-graphs for the simpler problems when the value of $x$
has been respectively fixed to $0$ and $1$ (or {\em false} and {\em true}).
The recursion ends when the formula is logically equivalent to either
the constant $\top$ or $\bot$, producing a leaf node in
the decision diagram.
This can be checked by solving the propositional satisfiability problem.

A branch in the search graph can be pruned based on the cost of an optimal decision diagram for a subformula (so that this decision diagram was found earlier in the search in an alternative branch for the same subformula).
We maintain a lower bound $\LB$ (initially $\infty$) and the cost of the thus-far constructed parts of the current decision diagram, in $\curcost$.
We prune those subgraphs of the current search node for which the sum of \curcost\ and the subdiagram cost exceeds $\LB$.
The value of $\LB$ is the lowest expected cost partial decision diagram that can be compared to the current one under construction.
Crucially, \LB\ is only passed to further recursive calls and not returned, ensuring that we are essentially comparing the cost of a complete decision diagram for the subformula at the point in which \LB\ is updated to the cost of an alternative decision diagram, still under construction, for the same subformula.
Pruning is done only when the latter is higher than $\LB$.

The subprocedure {\em assign}($l,\phi$), where $l$ is a literal ($x$ or $\neg x$ for some $x\in X$), replaces occurrences of $x$ by the constant $\top$ (if $l=x$) or the constant $\bot$ (if $l=\neg x$) and unit propagation-like simplifications are employed.\footnote{Our implementation represents the formula in conjunctive normal form (CNF), and {\em assign} is implemented by adding $l$ to the clause set as a unit clause, exhaustively applying the unit resolution and unit subsumption rules, to eliminate any clause that is a strict superset of another clause, and finally removing the unit clause $l$. This simplifies the clause set so that occurrences of $l$ are correctly removed and any resulting unit clauses are used to simplify the formula further.}
We next provide an illustrative example.

\begin{algorithm*}[t]
    \caption{Optimal Expected-Cost Decision Diagram Search}
    \label{alg:main}
    \begin{algorithmic}[1]
    \Function{OptCost}{$\phi, B, \curcost, \LB, \prob$}
        \Comment{$\phi$: CNF, $B$: valuation, $\prob$: path probability}
        
        \If{$B \in \text{Cache}$}
            \Return $\prob\cdot \text{Cache}[B]$ \Comment{Retrieve cached subsolution} \label{li:cache_check}
        \EndIf
        
        \If{$\phi$ is either valid or unsatisfiable} \Return $0$ \Comment{Terminal node of the decision diagram} \EndIf
        
        \State $\mincost \gets \infty$
        
        \For{each variable $x$ occurring in $\phi$, ordered by heuristic} \label{li:variteration}
            \State $\selcost \gets \prob \cdot \cost(x)$ \Comment{Immediate observation cost} \label{li:cost1}
            
            \If{$\curcost + \selcost \geq \LB$} \continue \Comment{Prune if observing $x$ is too costly}\EndIf \label{li:prune1} 
            
                \State $\selcost \gets \selcost + $ \Call{OptCost}{$assign(x{=}\top,\phi), B\cup\{x{=}\top\}, \curcost+\selcost, \LB, \prob\cdot P(x\mid B)$} \label{li:cost2}
                \If{$\curcost + \selcost \geq \LB$} \textbf{continue} \Comment{Prune if observing $x$ and the first truth value is too costly}\EndIf \label{li:prune2} 

                \State $\selcost \gets \selcost+ $ \Call{OptCost}{$assign(x{=}\bot,\phi), B\cup\{x{=}\bot\}, \curcost+\selcost, \LB, \prob\cdot P(\neg x\mid B)$} \label{li:cost3}
                \If{$\curcost + \selcost \geq \LB$} \textbf{continue} \Comment{Prune if observing $x$ and both truth values is too costly}\EndIf \label{li:prune3} 
            
            \State $\mincost \gets \min(\mincost,\selcost)$ \label{li:update_mincost}
            \State $\LB\gets \min(\LB, \curcost+\mincost)$ \label{li:update_bound}
        \EndFor
        
        \If{$\mincost < \infty$}
            \State $\text{Cache}[B] \gets \mincost/\prob$ \Comment{Store (normalized) subsolution} \label{li:cache_add}
        \EndIf
        \State \Return $\mincost$
    \EndFunction
    \end{algorithmic}
\end{algorithm*}

\begin{example}
  Consider the formula $F=(A\disj B\disj C)\conj(\neg A\disj D)$, with a uniform distribution over valuations, and with cost 10 for observing $C$ and cost 1 for other variables.
  Part of the search tree for finding a decision diagram for $F$ is shown in Figure~\ref{fi:dtreesearchexample},
  including finding the first two decision diagrams.
  The search tree has variable nodes (with a diamond shape) with two subtrees corresponding to setting the variable true and false (with the {\em false} subtree indicated by the dashed line), and formula nodes that indicate the formula simplified with the variable values on the path from the root to the node.

  From the search tree, each decision diagram can be obtained by starting from the root node, and for each formula node, eliminating the formula node and all of its subtrees except one, at the same time connecting the remaining subtree with the parent of the formula node.
    For example, the first decision diagram that would be found on a left-to-right traversal is the tree shown on the left in Figure~\ref{fi:dt_ex}, with expected cost 4.5.

In our tree-search algorithm, at the formula node $B\disj C$ after traversing the subtree corresponding to $B$ and before proceeding to the subtree corresponding to $C$, we have the cost $4.5$ as the bound to which we would be comparing any decision tree obtained by searching the subtree for $C$. Now, when going down to the subtree for $C$, the cost so far would be $1 + 1\cdot 0.5 + 10 \cdot 0.5 = 6.5$, which exceeds the bound, indicating that any decision tree found by completing the search for the subtree $C$ would have a cost of at least $6.5$. Hence the search-tree can be pruned at that point, without having to go any deeper.
Thus, 4.5 is the cost of the best decision tree found by observing the variable $A$ first.
The search proceeds from the root node by sequentially considering $B,C$ and $D$ as the first observation.
\end{example}

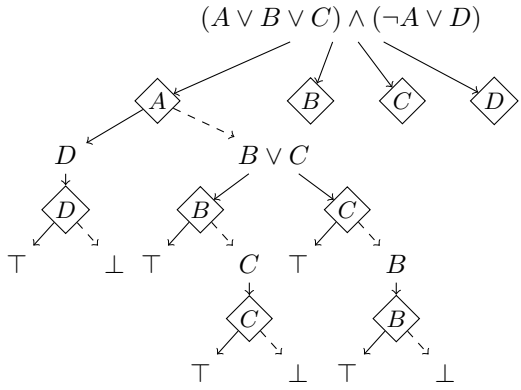
\begin{figure}[t]
\centering
\begin{tikzpicture}[xscale=0.9, yscale=0.8, scale=0.9]

  \node(R) at (5.5,5.5) {$(A\disj B\disj C)\conj(\neg A\disj D)$};

  \node[diamond,draw,inner sep=1pt] (A) at (2.5,4) {\small$A$};
  \node[diamond,draw,inner sep=1pt] (B) at (5,4) {\small$B$};
  \node[diamond,draw,inner sep=1pt] (C) at (6.5,4) {\small$C$};
  \node[diamond,draw,inner sep=1pt] (D) at (8,4) {\small$D$};

  \node(At) at (1,3) {$D$};
  \node(Af) at (4.4,3) {$B\disj C$};

  \node[diamond,draw,inner sep=1pt] (AtD) at (1,2) {\small$D$};
  \node[diamond,draw,inner sep=1pt] (AfB) at (3.2,2) {\small$B$};
  \node[diamond,draw,inner sep=1pt] (AfC) at (5.6,2) {\small$C$};

  \node(AtDt) at (0.2,1) {$\top$};
  \node(AtDf) at (1.8,1) {$\bot$};

  \node(AfBt) at (2.4,1) {$\top$};
  \node(AfBf) at (4.0,1) {$C$};
  \node(AfCt) at (4.8,1) {$\top$};
  \node(AfCf) at (6.4,1) {$B$};

  \node[diamond,draw,inner sep=1pt] (AfBfC) at (4.0,0) {\small$C$};
  \node[diamond,draw,inner sep=1pt] (AfCfB) at (6.4,0) {\small$B$};

  \node(AfBfCt) at (3.2,-1) {$\top$};
  \node(AfBfCf) at (4.8,-1) {$\bot$};
  \node(AfCfBt) at (5.6,-1) {$\top$};
  \node(AfCfBf) at (7.2,-1) {$\bot$};

  \tikzstyle{every path}=[arrows=-stealth,->]

  \draw(R) -- (A);
  \draw(R) -- (B);
  \draw(R) -- (C);
  \draw(R) -- (D);

  \draw(A) -- (At);
  \draw[dashed] (A) -- (Af);

  \draw(At) -- (AtD);

  \draw(AtD) -- (AtDt);
  \draw[dashed] (AtD) -- (AtDf);

  \draw(Af) -- (AfB);
  \draw(Af) -- (AfC);

  \draw(AfB) -- (AfBt);
  \draw[dashed] (AfB) -- (AfBf);

  \draw(AfC) -- (AfCt);
  \draw[dashed] (AfC) -- (AfCf);

  \draw(AfBf) -- (AfBfC);
  \draw(AfCf) -- (AfCfB);

  \draw(AfBfC) -- (AfBfCt);
  \draw[dashed] (AfBfC) -- (AfBfCf);

  \draw(AfCfB) -- (AfCfBt);
  \draw[dashed] (AfCfB) -- (AfCfBf);

\end{tikzpicture}
\caption{Part of the search tree for $(A\disj B\disj C)\conj(\neg A\disj D)$}\label{fi:dtreesearchexample}
\end{figure}

Note that line~\ref{li:prune3} might seem superfluous, as it does not actually prune the search tree.
This is not the case.
Its purpose is to prevent updating {\em childcost} and caching the value $C_x$ for the current sub-graph when the lowest-cost evaluation of assign($x=\bot$,$\phi$) was never found due to pruning inside the preceding recursive call to {\sc OptCost} (counterexample can be found in the Appendix).

Algorithm~\ref{alg:main} caches the normalized optimal cost of each solved subproblem (line~\ref{li:cache_add}), so that the cost can be rescaled by the path probability when the same valuation state is reached again (line~\ref{li:cache_check}).
Appendix~\ref{sec:advanced_cache} gives the complete version with lower-bound caching for pruned searches.

The selection order of variables also impacts effectiveness.
We would like to find the most promising, low-cost decision diagrams first, so that other branches can be pruned quickly.
Similarly, we would like to first traverse the more promising branch, for {\em true} or {\em false}, instead of trying them in fixed order.
We detail a number of ordering heuristics below.

Pruning could be further strengthened.
Our algorithm only considers the {\em parentcost}, the cost of the current variable $\cost(x)$, and the cost of the first sub-tree with $x$ assigned {\em true} or \emph{false} (lines~\ref{li:prune1}, \ref{li:prune2} and \ref{li:prune3}).
Lower bounds on the cost of the untraversed subtrees (corresponding to the formulas $\mathit{assign}(x,\phi)$ and $\mathit{assign}(\neg x,\phi)$) could be used to strengthen the pruning.
Doing this efficiently is not trivial, however.
One could, e.g., append to \LB\ a cost corresponding to the cheapest variable occurring in one of the truth value branches.
However, this could overestimate the lower bound, since $\mathit{assign}(x,\phi)$ could already be unsatisfiable or valid; a full satisfiability check on $\mathit{assign}(x,\phi)$ might be required for a correct bound.
We leave the investigation of stronger bounds to future work.

\subsection{Heuristics for Search Order}\label{se:heuristics}

\newcommand{\score}{\textsc{Score}}
\newcommand{\scoreb}{\textsc{BranchScore}}
\newcommand{\selectioncost}{\textsc{VScore}}
\newcommand{\multipliercost}{\ensuremath{m_c}}
\newcommand{\multipliervar}{\ensuremath{m_v}}
\newcommand{\multiplier}{\ensuremath{m}}

We propose heuristics for choosing the order of evaluating variables (see line~\ref{li:variteration} of Algorithm~\ref{alg:main}).

It is preferable to first find lower cost subsolutions so that pruning is done earlier, reducing the amount of searching.
To select the next variable for branching, the algorithm considers variables that appear in the current formula and scores them using a heuristic function $\score$.
The evaluation is based on the cost of observing a given variable as well as properties of the subformulas resulting from assigning that variable either true or false.

We assume input formulas to be in CNF.
Recall that $\cost(x)$ gives the cost of observing variable $x\in X$, while $P(x)$ represents the prior probability that $x$ evaluates to true.
For each candidate variable $x$, the score is computed as:
\begin{align*}
    \score(x) & =\ \selectioncost(x) \\
    & + P(x \mid B) \cdot \scoreb(\phi_x, h) \\
& + P(\lnot x \mid B) \cdot \scoreb(\phi_{\lnot x}, h)
\end{align*}
Here, $\phi_x$ and $\phi_{\lnot x}$ denote the CNF formulas after simplifying $\phi$ under valuations $v(x) = \top$ and $v(x) = \bot$ respectively, and $B$ is the current partial valuation.
The variable with the lowest score is selected.
The score uses both the direct cost of observing a given variable (via \selectioncost) and the estimated cost of the subsequent subproblems (via \scoreb).

We use $\cost(x)$ for \selectioncost, and propose three heuristics to compute $\scoreb(\phi, h)$.
The definitions of each heuristic are given in Table~\ref{tab:heuristics}.

Intuitively, \texttt{h\_cost} selects the lowest cost variable regardless of the rest of the formula, \texttt{h\_varCount} eliminates as many variables as possible from the formula, and \texttt{h\_totalCost} balances the immediate, certain cost of observing $x$ with a (rough) estimate of the cost of observing the rest of the variables.

For \texttt{h\_totalCost}, $\selectioncost$ is multiplied by a user-provided constant \multiplier; 0.5 in our experiments.
This is to balance between $\selectioncost$ and $\scoreb$.
The multiplier should be between 0 and 1, with a lower value weighting the certain observation cost $\selectioncost$ over future, uncertain costs.

We also consider \texttt{base}, a baseline heuristic with arbitrary variable ordering.

We use $\score$ to also select the truth branch to evaluate first.
Specifically, if $P(x \mid B) \cdot \scoreb(\phi_x, h) < P(\lnot x \mid B) \cdot \scoreb(\phi_{\lnot x}, h)$, then the lines~\ref{li:cost2}--\ref{li:prune2} are evaluated before lines~\ref{li:cost3}--\ref{li:prune3}.

\begin{table}[t]
\centering

\begin{tabular}{lrr}
\toprule
    \textbf{Name} & \textbf{\selectioncost} & \textbf{\scoreb} \\
\midrule
    \texttt{base}       & $0$ & $0$ \\
    \texttt{h\_cost}       & $\cost(x)$ & $0$ \\
    \texttt{h\_varCount}   & $\cost(x)$ & $N$ \\
    \texttt{h\_totalCost}  & $\cost(x)$ & $\multiplier\cdot \sum_{v \in \phi} \cost(v)$ \\

\bottomrule
\end{tabular}

\caption{
    Heuristic scoring strategies used in $\score(\phi)$.
}
\label{tab:heuristics}
\end{table}

\subsection{Correctness of the Algorithm}
\label{sec:correctness}

We show that Algorithm~\ref{alg:main} identifies the optimal cost.

\begin{theorem}[Correctness]
    \label{th:correctness_sketch}
    The least-cost decision diagram for $\phi$ has expected cost that equals the value returned by \textsc{OptCost}$(\phi,\emptyset,0,\infty,1.0)$ (Algorithm~\ref{alg:main}).
\end{theorem}

\begin{proof}
    Consider the algorithm without pruning and caching, i.e. remove lines~\ref{li:prune1}, \ref{li:prune2} and \ref{li:prune3} as well as line~\ref{li:cache_check}.
    Then the algorithm simply exhaustively searches over all decision diagrams where each leaf corresponds to an valuation under which $\phi$ is either valid or unsatisfiable.
    In each stage, the cost of observing the variables occurring in the current valuation (discounted by its probability) is stored in variable $\curcost$.
    The cost of the minimum cost decision diagram for the working formula $\phi$ is collected in the variable $\mincost$.
    In the base case, i.e. when $\phi$ is valid or unsatisfiable, a cost of 0 is incurred.
    In other nodes, the total cost comprises the probability of visiting this node given the current search path, $p$, times the cost of observing the considered variable $x$ (line~\ref{li:cost1}), plus the costs of the subtrees where $x$ is true and false, respectively (lines~\ref{li:cost2} and \ref{li:cost3}).
    In the main loop, the cost of the optimal decision diagram for $\phi$ for each unobserved variable is compared, and the lowest is selected as $\mincost$ (line~\ref{li:update_mincost}); the optimal cost of determining the truth value of the working formula $\phi$ is returned at the end.

    Pruning preserves correctness.
    To see this, consider that $\LB$ is only updated after both true/false branches are evaluated in a depth-first manner (line~\ref{li:update_bound}) and it is only passed down in further recursive calls, not returned.
    Consider a call to \textsc{OptCost} in which $\LB$ is updated and let the working formula at this point be $\phi_b$, and the $\curcost$ and $\mincost$ at this point be $\curcost_b$ and $\mincost_b$, respectively.
    We claim that the return value for this call is the same whether pruning happens or not; this implies the overall claim, since \LB\ is updated at least after the first variable at the root is evaluated.
    To see this, note that $\mincost_b=\LB-\curcost_b$ is the cost of the lowest cost decision diagram found for $\phi_b$ so far (divided by $p$, which is the same for all variables on this search node).
    Assume that sometime after selecting another variable $x$ in the same loop, we prune on line~\ref{li:prune1}, i.e. $\curcost+\selcost \geq \LB$.
    Thus the cost for any decision diagram for $\phi_b$ that would be found in the pruned branch is at least $\curcost+\selcost-\curcost_b$. Since line~\ref{li:prune1} prunes only when $\curcost+\selcost\geq\LB$, this cost is at least $\LB-\curcost_b=\mincost_b$.
    
    Thus, any decision diagram that would be found after the pruning happened would not contribute to lowering $\mincost_b$ (at the recursive call to \textsc{OptCost} where $\LB$ was found) and therefore also not to lowering the final $\mincost$ at the root.
    Similar reasoning holds for pruning at lines~\ref{li:prune2} and \ref{li:prune3}.

    Finally, exact caching does not affect correctness.
    Exact cache entries store only optimal subsolutions (line~\ref{li:cache_add}), so retrieving an exact entry (line~\ref{li:cache_check}) returns the normalized optimal cost multiplied by the probability of the current path.
    Therefore, caching can only avoid repeated work and cannot change the optimal cost returned by the root call. The advanced lower-bound caching version in Appendix~\ref{sec:advanced_cache} uses lower-bound entries only to prune, never as exact costs.
\end{proof}

\section{Empirical Evaluation}
\label{sec:experiments}
We evaluate our approach in four ways.
First, we show runtime and memory performance on hard instances.
Next, we investigate how the structure of the underlying propositional formula affects runtime.
We then examine the effect of relaxing optimality on runtime and solution quality with a beam-search version of our algorithm.
Finally, we evaluate a structured real-world instance.

Our implementation of Algorithm~\ref{alg:main} (Python~3.9, \texttt{python-sat} v1.8~\cite{itk-sat24}, CaDiCaL~1.9.5~\cite{fleury2020cadical}) was used for all experiments. Experiments were conducted on 2.5 GHz Intel Xeon E5-2680 v3 CPUs with a 120-minute time limit and a 16 GB memory limit per instance.

We are not aware of a suitable benchmark set for SBFE and thus
generated a set of random 3-SAT formulas~\cite{MitchellSL92} along with costs and probabilities. Random 3-SAT is parameterized by the number of variables $N$ and clauses $M$.

We drew variable observation costs uniformly at random from the integers $1$ to $10$.
For simplicity, we drew independent probabilities for truth values uniformly at random.

Figure~\ref{fig:scale_time_vs_nodes} presents results with respect to runtime (left) and peak memory usage (right) on datasets with $6$--$15$ variables and a fixed clause-to-variable ratio of $1.5$ which generates hard instances, averaged over 100 instances per size.
We evaluate the proposed heuristics and four baseline variants: suffixes \texttt{\_cache} and \texttt{\_noPruning} denote enabled caching and disabled pruning, respectively, while \texttt{base} is pruning without caching.
Lines are truncated at the first occurrence of a timeout.
As expected, both runtime and memory usage grow exponentially with problem size.
Both pruning and caching yield exponential improvements over the baselines.
Only cached variants successfully scale to 15 variables.
The best performing methods are \texttt{h\_varCount\_cache} and \texttt{h\_totalCost\_cache}, with an average runtime of 152 and 182 seconds, respectively, on instances with 15 variables, compared to 275, 325 and 518 seconds for the three other caching-enabled variants.

In contrast to the significant memory use shown for the cached variants, without caching, memory usage is minimal, and thus in memory-constrained settings, forgoing caching might be sensible.

\begin{figure}[t]
\centering
\includegraphics[width=\columnwidth]{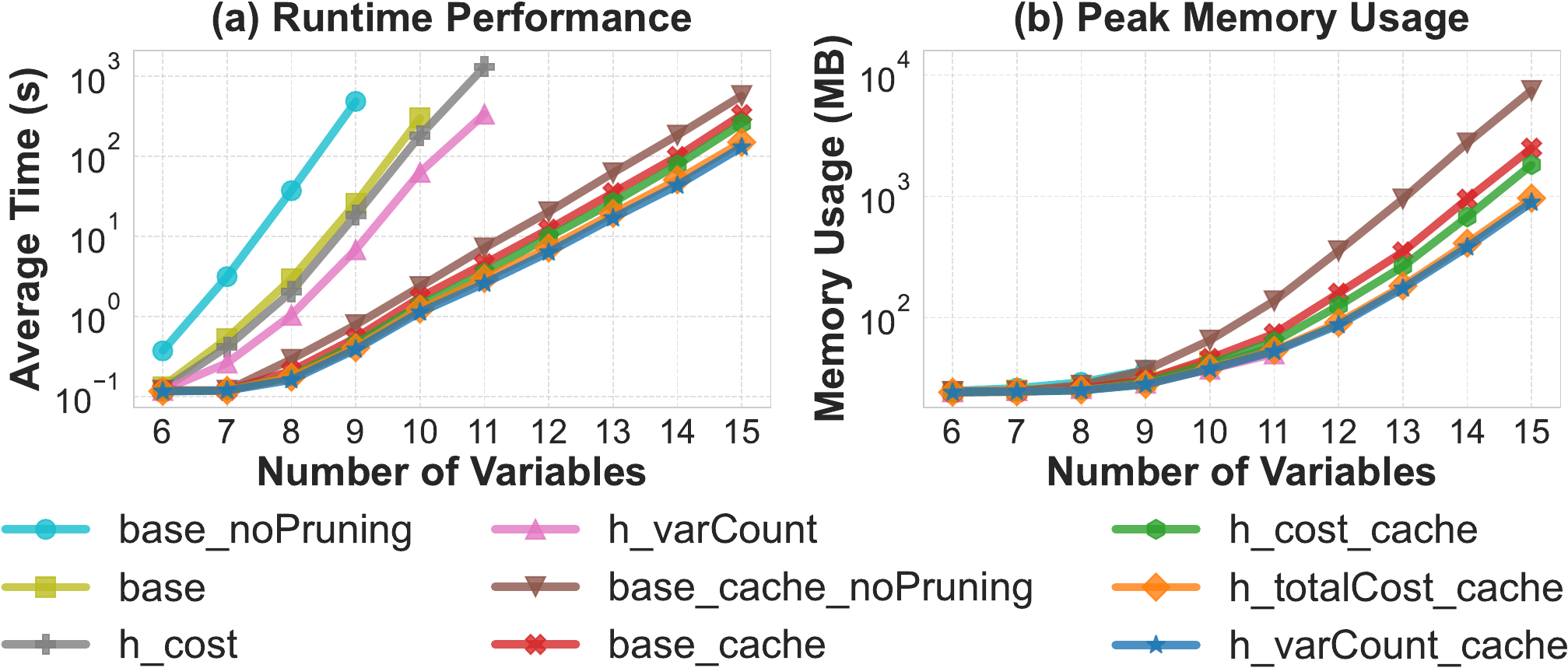}
\caption{Average runtime (left) and peak memory usage (right) of the evaluated configurations as the number of variables increases. Curves end at the first timeout.}
\label{fig:scale_time_vs_nodes}
\end{figure}

\begin{figure}[t]
\centering
\includegraphics[width=0.7\columnwidth]{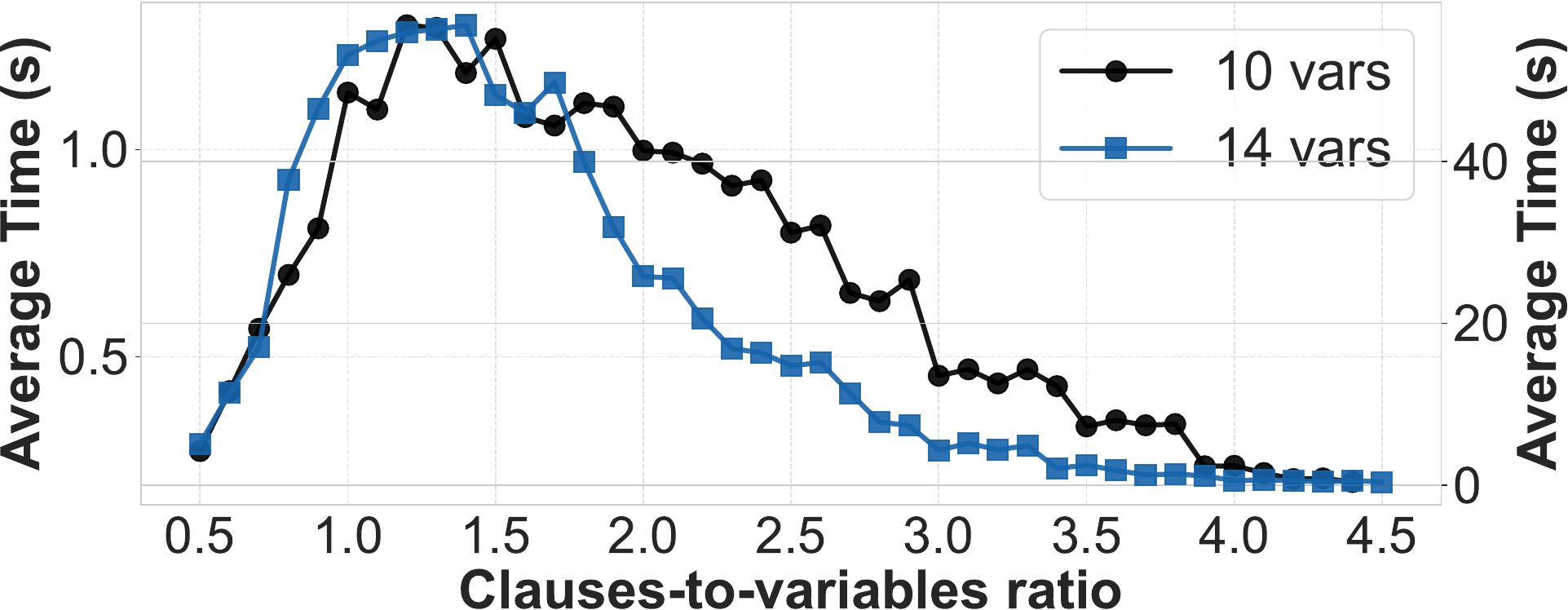}

    \caption{Impact of the clause-to-variable ratio on runtime for $N=10$ (left y-axis) and $N=14$ (right y-axis).}
\label{fig:ratio_time}
\end{figure}

We proceed to the impact of the clause-to-variable ratio ($r$) on problem difficulty.
We show in Figure~\ref{fig:ratio_time} the performance of the algorithm on instances with varying ratios
for two different problem sizes ($N=10$ and $N=14$).
For both instance sizes, the average solving time peaks around a ratio of $1.2$ to $1.5$, well before the SAT phase transition point ($\approx$4.27).
This is likely because in highly constrained instances there are lots of logical consequences and thus smaller trees, while less constrained instances are easily satisfiable, so setting a few variables suffices.

Thus, $r=1.5$ is a near-worst-case (see above) and $r=3$ an easier scenario (below).

\begin{figure}[t]
\centering
\includegraphics[width=\columnwidth]{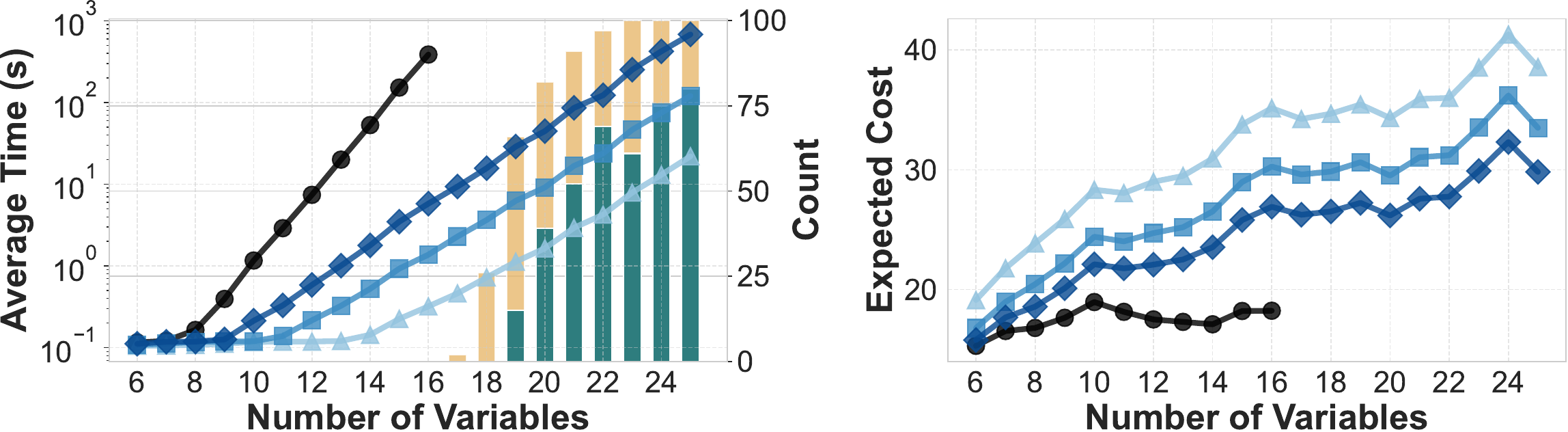}
\includegraphics[width=\columnwidth]{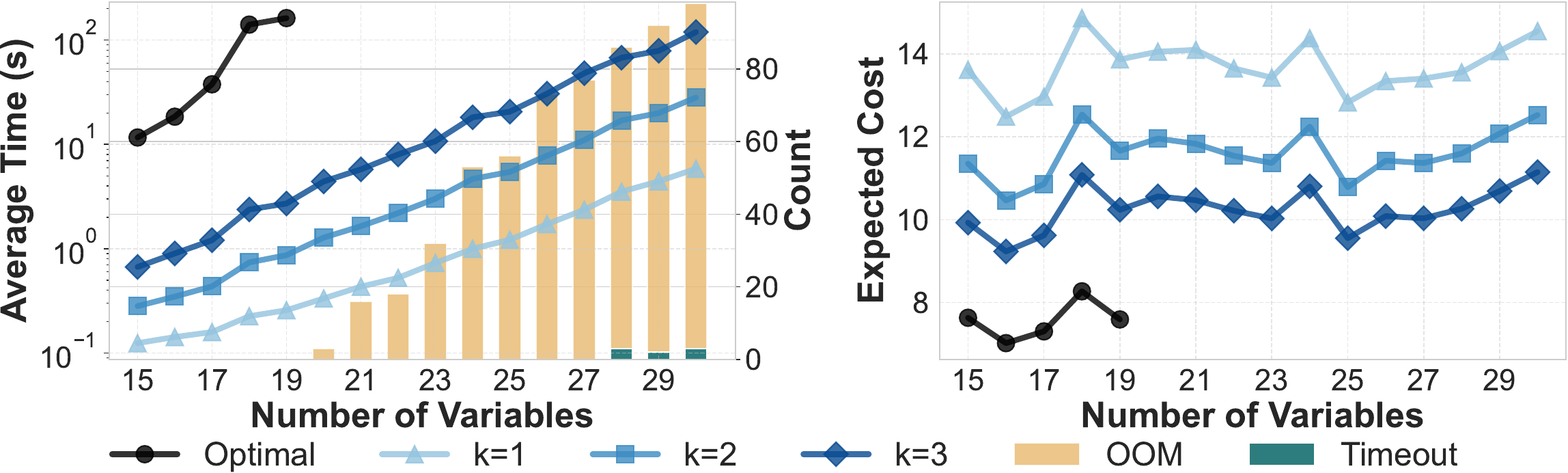}
\caption{Runtime (left) and expected cost (right) for exact versus beam search ($k \in \{1,2,3\}$) on hard ($r=1.5$, top) and relaxed ($r=3.0$, bottom) instances. OOMs and timeouts are shown for the exact algorithm only.}
\label{fig:beam}
\end{figure}

Finally, we compare greedy beam search (branching into only the best $k$ variables at each search node) to our exact algorithm (Figure~\ref{fig:beam}).
For this comparison, the exact baseline is the best variant from Figure~\ref{fig:scale_time_vs_nodes}, i.e. \texttt{h\_varCount\_cache}.
In the harder instances with $r=1.5$ (top), the exact algorithm exceeded resource limits when $n>16$, whereas beam search solved all instances.
At $n=16$, beam search is more than an order of magnitude faster.

In terms of cost, beam with $k=3$ remains within $\approx 30\%$ of the optimal.
On instances with $r=3$ (bottom), the exact algorithm scales further to $n=19$, with runtime and cost behaviour similar to that of beam search.

Interestingly, with higher clause-to-variable ratio, the optimal algorithm mostly runs out of memory rather than time, while the opposite is true with a lower ratio.

\paragraph{Illustrative medical-diagnosis case study}

To illustrate how a learned classifier can be converted into a structured SBFE instance, we transform a medical-diagnosis classification tree into a stochastic Boolean function evaluation problem.
We trained a decision tree~\cite{pedregosa2011scikit} on the \textit{Cleveland Heart Disease} dataset~\cite{heart_disease_45} to capture the diagnostic rules.
We turned the rules into a Boolean formula as implications from the conditions of each path to the class at the leaf.
Continuous features were discretized via threshold encoding, where consistency was enforced by generating implication constraints (e.g., $\neg(f \le t_1) \lor (f \le t_2)$ for thresholds $t_1 < t_2$) to prevent impossible values.
Each threshold predicate inherits the cost of its underlying medical test from Turney~\shortcite{turney1994cost}.
We set each predicate probability to its empirical marginal frequency and evaluate both policies under the resulting independent product model, which is the probability model used by our algorithm.
This process yielded a Boolean formula with 16 variables and 22 clauses.
Under this common model, evaluating the original decision tree costs \$236.4, while our algorithm identified the global optimum of \$137.4 in 819 seconds, a 41.9\% reduction in expected predicate-acquisition cost.

\section{Conclusion}
\label{sec:conclusions}

We have considered a setting in which a Boolean formula represents a condition for a decision, where observing variables incurs different costs and valuations follow a joint probability distribution.
We have presented a branch-and-bound algorithm for finding a decision diagram of minimum expected cost in this setting.
We have introduced pruning, caching, and variable-selection heuristics, and empirically demonstrated their effectiveness.
We have shown that a greedy beam-search variant offers a controllable efficiency--quality trade-off.
Finally, we have analyzed the computational complexity of decision-diagram construction, showing that it is $\sharpP$-hard and contained in $\PSPACE$.

\bibliographystyle{named}

\bibliography{ourbib,abbrvs}

\clearpage
\appendix
\input{appendix-body}

\end{document}

%% file: appendix-body.tex
\section{Pruning in the Algorithm Pseudo-Code}

In the pseudo-code of our algorithm given in the main paper,
we have three code lines like this one.
\begin{equation}\label{codeline}
  \mathbf{if} \mathit{ parentcost} + C_x\geq \mathit{bound}\ \mathbf{then}\ \mathbf{continue}
\end{equation}
The first two code lines obviously prune parts of the search tree,
but the third one obviously does not.
The purpose of the third code line is instead to prevent caching of
values that do not represent the expected cost of the best decision
diagram for a given formula.
The next example illustrates what would happen without this code line.

\begin{example}
Consider a decision diagram construction problem in which
at some stage the current formula is
\[ (Z\impl Y)\conj(\neg Z\impl(X\disj Y)) \]
and a part of the search tree looks as in Figure~\ref{fi:pruningproblem}.

We show how the algorithm without the pruning condition (\ref{codeline}) can cache a cost for the formula $X\disj Y$ that is not the lowest possible cost for a decision diagram for $X\disj Y$.

The core issue is that caching makes it possible to avoid some prunings that would otherwise take place, and in some cases for some formulas, a non-optimal diagram does not get pruned, but the optimal diagram for that same formula will get pruned, and the non-optimal value gets stored in the cache.

Let all the variables be statistically independent and $P(X)=P(Y)=0.5$.

For the problematic behavior to be possible in this example,
the decision diagram for $X\disj Y$ with root $Y$ must have a lower cost
than the diagram for $X\disj Y$ with root $X$.
This will be the case when
\[ c(Y)+0.5c(X) < c(X)+0.5c(Y). \]

For the pruning to not happen for the diagram with root $X$ and for it to happen for the diagram with root $Y$, it must be that
\[ c(X) < \mathit{bound} \ \mathrm{ and }\  c(Y)+0.5c(X)>\mathit{bound}. \]

These hold, for example, when the costs and the effective bound at the relevant spot in the search tree are as follows.
  \[
  \begin{array}{rcl}
    c(X) & = & 1.2 \\
    c(Y) & = & 1 \\
    \mathit{bound} & = & 1.4
  \end{array}
  \]

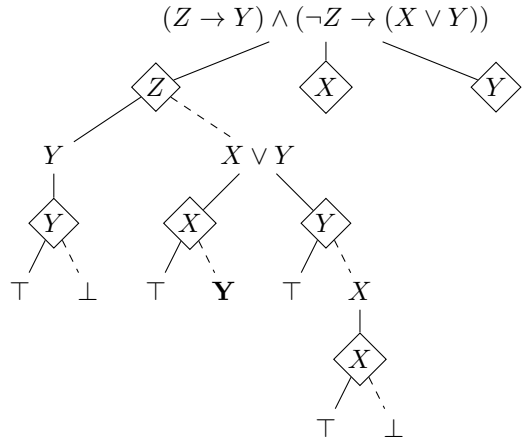
\begin{figure}
\begin{center}
\begin{tikzpicture}[xscale=0.45,yscale=0.45]

  \node(root) at (10,10) {$(Z\impl Y)\conj(\neg Z\impl(X\disj Y))$};
  \node[diamond,draw,inner sep=1pt] (Zt) at (5,8) {$Z$};
  \node[diamond,draw,inner sep=1pt] (Xt) at (10,8) {$X$};
  \node[diamond,draw,inner sep=1pt] (Yt) at (15,8) {$Y$};

  \draw(root) -- (Zt);
  \draw(root) -- (Xt);
  \draw(root) -- (Yt);

  \node(ZtTf) at (2,6) {$Y$};
  \node[diamond,draw,inner sep=1pt] (ZtT) at (2,4) {$Y$};

  \node(ZtF) at (8,6) {$X\disj Y$};

  \draw(Zt) -- (ZtTf);
  \draw(ZtTf) -- (ZtT);
  \draw[dashed] (Zt) -- (ZtF);

  \node(ZtT1) at (1,2) {$\top$};
  \node(ZtT0) at (3,2) {$\bot$};

  \draw(ZtT) -- (ZtT1);
  \draw[dashed] (ZtT) -- (ZtT0);

  \node[diamond,draw,inner sep=1pt] (ZtFX) at (6,4) {$X$};
  \node[diamond,draw,inner sep=1pt] (ZtFY) at (10,4) {$Y$};

  \draw(ZtF) -- (ZtFX);
  \draw(ZtF) -- (ZtFY);

  \node(ZtFX1) at (5,2) {$\top$};

  \node(ZtFX0f) at (7,2) {$\mathbf{Y}$};

  \draw(ZtFX) -- (ZtFX1);
  \draw[dashed] (ZtFX) -- (ZtFX0f);

  \node(ZtFY1) at (9,2) {$\top$};

  \node (ZtFY0f) at (11,2) {$X$};
  \node[diamond,draw,inner sep=1pt] (ZtFY0) at (11,0) {$X$};

  \draw(ZtFY) -- (ZtFY1);
  \draw[dashed] (ZtFY) -- (ZtFY0f);
  \draw(ZtFY0f) -- (ZtFY0);

  \node(ZtTFY0X1) at (10,-2) {$\top$};
  \node(ZtTFY0X0) at (12,-2) {$\bot$};

  \draw(ZtFY0) -- (ZtTFY0X1);
  \draw[dashed] (ZtFY0) -- (ZtTFY0X0);

\end{tikzpicture}
\end{center}
\caption{Example of problematic pruning}\label{fi:pruningproblem}
\end{figure}

What happens in this example is the following.

\begin{enumerate}
\item At the first (leftmost) $Y$ node, the cost of the diagram for formula $Y$ is cached.
\item At the second formula $Y$ node (marked in {\bf boldface}), the cost of the subtree is fetched from the cache. The bound of the $X\disj Y$ node is violated, but the cached value is used as if there had been no bound violation.
  For $X\disj Y$ the diagram with root $X$ is now the best one so far.
\item The search for the best diagram for $X\disj Y$ proceeds to the sibling with $Y$ as the root. The illustration shows the full diagram below that node, but the second sub-tree with $Y=0$ will not be developed beyond the node with formula $X$ due to a violation of the bound.
\item The issue now is that this sub-tree, with $Y$ in the root, has a lower cost than the sub-tree with $X$ in the root. But this sub-tree was pruned. Hence the sub-optimal value for $X\disj Y$, which was found first, gets cached.
\item Later in the search, if/when a node with $X\disj Y$ is encountered again,
  this sub-optimal cost for $X\disj Y$ will be retrieved from the cache, which can lead to the algorithm delivering a decision diagram that is not optimal.
\end{enumerate}
\end{example}

To prevent sub-optimal values from being cached, we discard any cost computation that may have depended on pruning an optimal sub-solution. We have chosen to use the third code line (\ref{codeline}). Another possibility would be to check the bound when retrieving a value from the cache.

\section{Full Algorithm with Advanced Caching}\label{sec:advanced_cache}
Algorithm~\ref{alg:advanced_cache} presents the complete advanced-caching implementation of our approach. While the simplified version in the main text only caches optimal subsolutions, this version incorporates an advanced mechanism to handle cases where the search is pruned (i.e., when $\mincost = \infty$). Note that the optimality still holds in this version.

\paragraph{Lower Bound Caching} When the search for a specific assignment $B$ is cut off because the cost exceeds the global bound $\LB$, we cannot store an exact cost. However, we have derived valid information: the cost of this subtree is at least the difference between the bound and the parent cost. In line~\ref{adv:cache_lb_store}, we store this lower bound (normalized by probability $\prob$) along with a flag \FlagLB.

\paragraph{Reusing Lower Bounds} When this state $B$ is revisited later with a potentially looser bound, we retrieve the cached lower bound (line~\ref{adv:cache_check}). Even though it is not an exact value, it allows us to perform an immediate consistency check. If the cached lower bound, when scaled by $\prob$ and added to the current $\curcost$, still exceeds the new $\LB$ (line~\ref{adv:cache_lb_check}), we can immediately prune the search (returning $\infty$). This effectively prevents redundant explorations of expensive subtrees even when their exact costs are unknown.

\begin{algorithm*}[t]
    \caption{Optimal Expected-Cost Decision Diagram Search with Advanced Caching}
    \label{alg:advanced_cache}
    \begin{algorithmic}[1]
    \Function{OptCost}{$\phi, B, \curcost, \LB, \prob$}
        \Comment{$\phi$: CNF, $B$: valuation, $\prob$: path probability}

        \If{$B \in \text{Cache}$}
            \State $(\hat{v}, \textit{flag}) \gets \text{Cache}[B]$ \label{adv:cache_check}
            \Comment{Retrieve cached subsolution}
            \If{$\textit{flag} = \FlagExact$} \Return $\hat{v} \cdot \prob$ \EndIf
            \If{$(\textit{flag} = \FlagLB)$ \textbf{and} $(\curcost + \hat{v} \cdot \prob \geq \LB)$} 
                \Return $\infty$ \label{adv:cache_lb_check}
            \EndIf
        \EndIf
        
        \If{$\phi$ is either valid or unsatisfiable} \Return $0$ \Comment{Terminal node of the decision diagram} \EndIf
        
        \State $\mincost \gets \infty$
        
        \For{each variable $x$ occurring in $\phi$, ordered by heuristic} \label{adv:variteration}
            \State $\selcost \gets \prob \cdot \cost(x)$ \Comment{Immediate observation cost} \label{adv:cost1}
            
            \If{$\curcost + \selcost \geq \LB$} \continue \Comment{Prune if observing $x$ is too costly}\EndIf \label{adv:prune1} 
            
                \State $\selcost \gets \selcost + $ \Call{OptCost}{$assign(x{=}\top,\phi), B\cup\{x{=}\top\}, \curcost+\selcost, \LB, \prob\cdot P(x\mid B)$} \label{adv:cost2}
                \If{$\curcost + \selcost \geq \LB$} \textbf{continue} \Comment{Prune if observing $x$ and the first truth value is too costly}\EndIf \label{adv:prune2} 

                \State $\selcost \gets \selcost+ $ \Call{OptCost}{$assign(x{=}\bot,\phi), B\cup\{x{=}\bot\}, \curcost+\selcost, \LB, \prob\cdot P(\neg x\mid B)$} \label{adv:cost3}
                \If{$\curcost + \selcost \geq \LB$} \textbf{continue} \Comment{Prune if observing $x$ and both truth values is too costly}\EndIf \label{adv:prune3} 
            
            \State $\mincost \gets \min(\mincost,\selcost)$ \label{adv:update_mincost}
            \State $\LB\gets \min(\LB, \curcost+\mincost)$ \label{adv:update_bound}
        \EndFor
        
            \If{$\mincost = \infty$}
                \State $\text{Cache}[B] \gets ((\LB - \curcost)/\prob, \FlagLB)$ \Return $\infty$ \Comment{Store non-optimal subsolution as lower bound} \label{adv:cache_lb_store}
            \Else
                \State $\text{Cache}[B] \gets (\mincost/\prob, \FlagExact)$
                \Return $\mincost$ \Comment{Store (normalized) subsolution} \label{adv:cache_add}
            \EndIf
    \EndFunction
    \end{algorithmic}
\end{algorithm*}

\section*{Further Details on Empirical Evaluation}

\section{Dataset Generation Details}

All reported random-instance results average over 100 independently generated 3-SAT formulas. For scalability, $N$ ranges from 6 to 15 at a fixed clause-to-variable ratio of 1.5. For the ratio study, $N\in\{10,14\}$ and $M/N$ ranges from 0.5 to 4.5 in increments of 0.1. Beam search uses the same instances as the scalability study and $k\in\{1,2,3\}$. Observation costs are drawn uniformly from the integers 1 to 10, and truth probabilities are independently sampled from $[0,1]$ and rounded to two decimal digits.

\section{Illustrative Medical-Diagnosis Case Study}

\subsection{Data and Proxy Model}
We use the Cleveland Heart Disease data from the UCI Heart Disease collection~\cite{heart_disease_45}. The target \texttt{num} is binarized as $\texttt{num}>0$, and missing feature values are imputed with the feature median. We train \texttt{DecisionTreeClassifier(max\_depth=4, random\_state=42)} on the resulting data to obtain a deterministic CART proxy. The proxy has 15 split predicates. We add one zero-cost target variable representing the binary output, yielding 16 Boolean variables in total.

\subsection{CNF Construction}
Each CART root-to-leaf path is converted into one clause encoding that the conjunction of its path conditions implies the label predicted at its leaf. The depth-4 proxy produces 16 such leaf-path clauses. To enforce consistency among predicates derived from the same physical feature, we add six threshold-consistency clauses of the form $\neg(f\le t_1)\lor(f\le t_2)$ for $t_1<t_2$. The resulting formula has 22 clauses.

\subsection{Probability and Cost Modeling}
For each split predicate $f\le t$, its probability is its empirical frequency in the cleaned data, clipped to $[0.001,0.999]$. The solver uses the product prior formed from these marginal probabilities. Each predicate inherits the Turney cost of its underlying physical feature~\shortcite{turney1994cost}, while the target variable has zero observation cost. The model charges additively for threshold predicates, including multiple thresholds of the same feature. This is an explicit modeling choice of the Boolean encoding, rather than a direct simulation of a clinical workflow.

\subsection{Like-for-Like Baseline and Scope}
We evaluate the original CART proxy recursively under the same product prior and the same additive predicate-cost model as the optimized decision diagram. Thus, the comparison of \$236.4 for the CART proxy with \$137.4 for the optimized diagram is a cost-optimization result for the encoded SBFE model. It is not a clinical validation or a claim about real diagnostic workflow costs.